\documentclass[10pt]{article}
\usepackage[preprint]{tmlr}
\usepackage[utf8]{inputenc}
\usepackage{amsmath,amssymb,amsthm,mathtools,bm}
\usepackage{booktabs,array,multirow}
\usepackage{graphicx}
\usepackage{float}
\usepackage{xcolor}
\usepackage{microtype}
\usepackage{enumitem}
\usepackage{url}
\usepackage{hyperref}
\usepackage{subcaption}
\usepackage[nameinlink,capitalise,noabbrev]{cleveref}
\usepackage{tikz}
\usepackage{xspace}
\usepackage{marvosym}
\usetikzlibrary{arrows.meta,positioning,fit}

\hypersetup{colorlinks=true,citecolor=blue!55!black,linkcolor=blue!55!black,urlcolor=blue!55!black}
\graphicspath{{figures/}}
\setlist[itemize]{leftmargin=*,topsep=2pt,itemsep=1pt}
\setlist[enumerate]{leftmargin=*,topsep=2pt,itemsep=1pt}

\newtheorem{proposition}{Proposition}
\newtheorem{lemma}{Lemma}
\newtheorem{definition}{Definition}
\newtheorem{assumption}{Assumption}

\newcommand{\E}{\mathbb{E}}
\newcommand{\cD}{\mathcal{D}}
\newcommand{\cR}{\mathcal{R}}
\newcommand{\cF}{\mathcal{F}}

\newcommand{\Risk}{\mathcal{L}}
\newcommand{\thetao}{\theta_{\mathrm{orig}}}
\newcommand{\thetar}{\theta_{\mathrm{retrain}}}
\newcommand{\thetau}{\theta_{\mathrm{unlearn}}}
\newcommand{\diff}{\operatorname{Diff}}
\newcommand{\norm}[1]{\left\lVert #1\right\rVert}
\newcommand{\abs}[1]{\left\lvert #1\right\rvert}
\newcommand{\avgdiff}{\textnormal{\textsc{Avg.\ Diff.}}\xspace}

\title{Do Influence-Derived Data Perturbations Enable Machine Unlearning?\\A Controlled Study of Three Plausible Roles}
\author{\name Chenkai Wu, \name Chrispine Kambimbi, \name Qinyang Zeng, \name Jun Yan (yanjun.ieee.org)\Letter \\ \addr Monash University, \addr LibrAI, \addr Tongji University, \addr Shanghai Ocean University}
\def\month{MM}
\def\year{YYYY}
\def\openreview{\url{https://openreview.net/forum?id=XXXX}}

\begin{document}
\maketitle

\begin{abstract}
We evaluate Deep Perturbation Learning (DPL), which perturbs training images and labels along influence-derived directions, in three roles in which prior work has positioned it for machine unlearning: a direct deletion signal (the strongest claim), a utility-preserving regularizer, and a warm start for adversarial unlearning. Evidence for the weaker roles has been used to support the stronger one, so we test each role separately under a matched protocol with exact-seed retraining baselines. An audit of the public implementation identifies two correctness issues: image directions are computed on augmented, normalized tensors but applied to raw images, and the label perturbation falls below float32 resolution, leaving labels unchanged. After correcting the image-perturbation pipeline, DPL fails the direct-deletion criterion on CIFAR-10/ResNet-18 in all three paired seeds. Its utility effects are inconsistent in sign across seeds, and once direction-computation time is counted it underperforms simple warm-start baselines. A one-seed Tiny ImageNet check likewise does not favor DPL as a regularizer or warm start; preprocessing inconsistencies in the released code make the direct comparison there inconclusive. These results cover random instance deletion only and do not rule out influence-based methods in other deletion regimes. We release a role-matched evaluation protocol and an audit checklist for perturbation-based deletion claims.
\end{abstract}

\section{Introduction}
\label{sec:intro}
Machine unlearning seeks to remove the influence of a designated subset of training data without paying the full cost of retraining \citep{cao2015towards,bourtoule2021machine}. For random instance deletion, exact retraining on the retained set is the relevant reference~\citep{ginart2019making,guo2020certified}. Because forgotten examples remain independent and identically distributed (i.i.d.) samples from the same distribution, a retrained classifier may still predict them correctly. Lower forget-set accuracy can therefore indicate a larger deviation from retraining rather than better unlearning.

Data perturbation changes the parameter trajectory through the inputs while leaving the architecture, objective, and optimizer unchanged. This can be useful when the training pipeline is fixed or the model is served in place. Deep Perturbation Learning (DPL)~\citep{song2023dpl} learns influence-derived image and label perturbations to improve supervised training. Its construction can be redirected at deletion: estimate the parameter displacement induced by removing a forget set $\cF$ and seek input perturbations whose gradient updates approximate it (\cref{sec:analysis}). However, the target displacement contains an inverse Hessian, and an input-space perturbation need not span the corresponding parameter-space direction.

This approximation could serve unlearning in three roles of decreasing strength, which we state as separate questions. First, can an influence-derived perturbation supply enough deletion pressure to act as a direct forgetting mechanism? Second, when an explicit forgetting objective is already present, can the perturbation preserve utility better than ordinary retained-data regularization? Third, can it initialize an adversarial-example unlearning method and reduce the cost of iterative attack generation? Each question requires its own control and its own predeclared pass/fail criterion, which we call a gate. An immediate increase in forget loss may disappear after model updating, preserved test accuracy may remain inconsistent with retraining under membership inference, and faster attack refinement may be offset by the cost of constructing the influence direction.

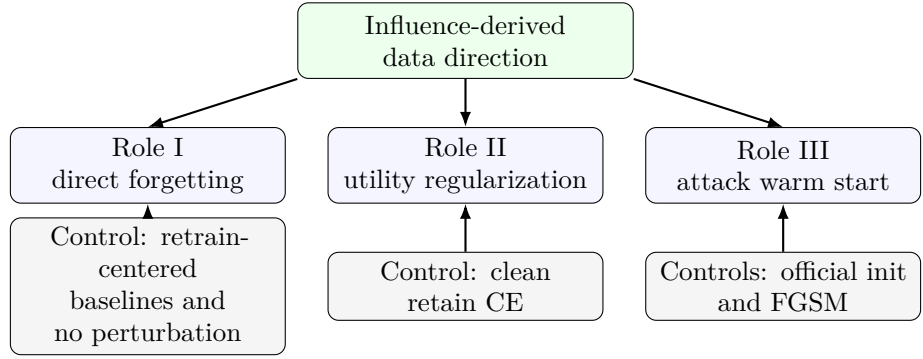
\begin{figure}[!t]
\centering
\begin{tikzpicture}[
  role/.style={draw,rounded corners,align=center,text width=34mm,minimum height=10mm,fill=blue!4},
  ctl/.style={draw,rounded corners,align=center,text width=34mm,minimum height=9mm,fill=gray!8},
  source/.style={draw,rounded corners,align=center,text width=42mm,minimum height=10mm,fill=green!7},
  arr/.style={-{Latex[length=2mm]},thick}
]
\node[source] (dir) at (0,0) {Influence-derived\\data direction};
\node[role] (direct) at (-4.2,-1.65) {Role I\\direct forgetting};
\node[role] (utility) at (0,-1.65) {Role II\\utility regularization};
\node[role] (warm) at (4.2,-1.65) {Role III\\attack warm start};
\node[ctl] (c1) at (-4.2,-3.25) {Control: retrain-centered\\baselines and no perturbation};
\node[ctl] (c2) at (0,-3.25) {Control: clean retain CE};
\node[ctl] (c3) at (4.2,-3.25) {Controls: official init\\and FGSM};
\draw[arr] (dir.south west) -- (direct.north);
\draw[arr] (dir.south) -- (utility.north);
\draw[arr] (dir.south east) -- (warm.north);
\draw[arr] (c1.north) -- (direct.south);
\draw[arr] (c2.north) -- (utility.south);
\draw[arr] (c3.north) -- (warm.south);
\end{tikzpicture}
\caption{Three plausible roles for an influence-derived data perturbation. Each role is evaluated against a simpler matched control. The study asks whether the perturbation contributes value beyond the surrounding optimization procedure.}
\label{fig:roles}
\end{figure}

\begin{figure}[!t]
\centering
\includegraphics[width=0.9\textwidth]{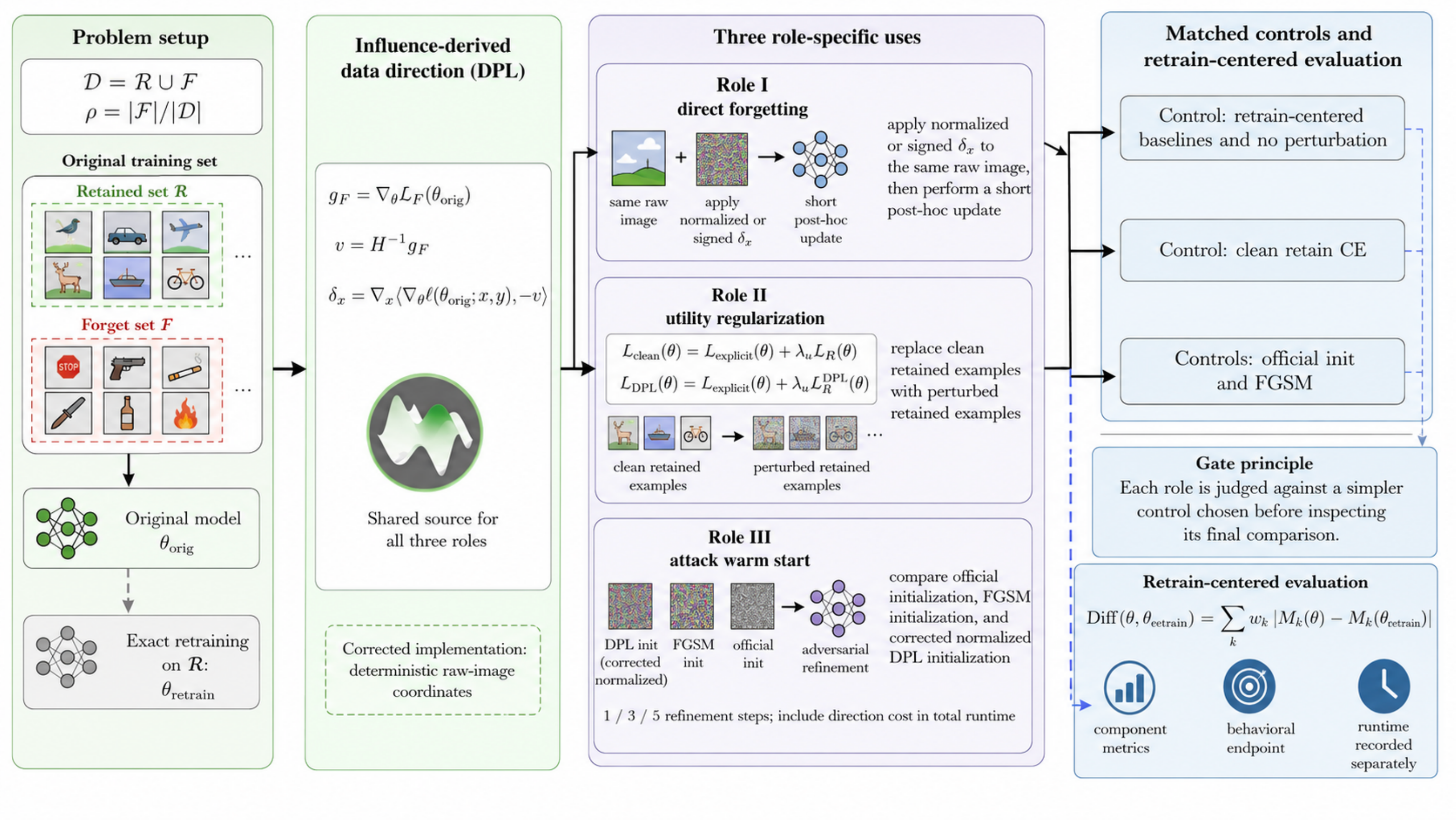}
\caption{Role-separated evaluation framework. A shared influence-derived direction is tested for direct forgetting, utility regularization, and attack warm starting, with exact retraining as the reference.}
\label{fig:framework}
\end{figure}

We separate the roles and test each against its matched baseline on CIFAR-10/ResNet-18 across three paired source/forget seeds (\cref{fig:roles}). Before any comparison, an audit of the implementation (\cref{sec:audit}) finds that image directions were computed on augmented, normalized tensors but applied to raw images, and that the label perturbation leaves float32 targets unchanged; we correct the image path and evaluate the corrected direction. For direct forgetting, the best image directions, selected optimistically over the diagnostic sweep, produce an immediate separation between forget and retain loss. This separation is either too small to matter or comes with substantial utility damage, the resulting model stays farther from retraining than either Saliency Unlearning (SalUn)~\citep{fan2024salun} or Adversarial Machine UNlearning (AMUN)~\citep{ebrahimpour2025amun}, and the direct-role gate fails in all three seeds. As a utility regularizer, DPL passes its gate in one of three seeds and does not reliably improve over clean retain-set cross-entropy. As a warm start, DPL behaves almost identically to a Fast Gradient Sign Method (FGSM) initialization~\citep{goodfellow2015explaining} while adding 47--66 seconds of direction-computation cost, and never beats both simple initializations under the full gate.

\Cref{fig:framework} shows the evaluation pipeline: one influence-derived direction is computed from the original model and tested in each role against its matched control, with exact retraining as the reference and task-specific metrics, behavioral outcomes, and computational cost reported throughout. A one-seed Tiny ImageNet~\citep{tinyimagenet} check with VGG-19 serves as an external implementation audit rather than a second decisive test; it exposes a preprocessing inconsistency in the official evaluation path that leaves the direct-role comparison there inconclusive.

This paper makes three contributions:
\begin{itemize}
    \item \textbf{Role-separated design and local analysis.} We assign a
    matched control and gate to each of the three roles and derive what an
    influence direction can and cannot support locally: why exact
    retraining retains high forget accuracy under random deletion, the
    first-order retraining displacement, the input direction that most
    improves a one-step parameter match, a condition under which a
    coordinate mismatch between where a direction is computed and where it
    is applied reverses the intended improvement, and why parameter
    alignment is only a surrogate for behavioral agreement.
    \item \textbf{Role-matched empirical evaluation.} With frozen
    configurations across three paired CIFAR-10 seeds, the corrected DPL
    direction fails the direct-deletion gate in every seed, passes the
    utility gate in one seed, and never beats both official and FGSM
    initialization as a warm start once direction cost is counted
    (\cref{sec:experiments}).
    \item \textbf{Implementation audit and checklist.} We identify a
    coordinate mismatch and a numerically inactive label perturbation in
    the direct path, a normalization/clamp inconsistency in the matched
    adversarial pipeline, and a Tiny ImageNet preprocessing defect that
    changes outcomes more than the choice of method (\cref{sec:audit}).
    \Cref{sec:takeaways} distills these into an audit checklist for
    perturbation-based deletion claims.
\end{itemize}
Our empirical claims are limited to the studied random-instance-deletion protocols. Within these protocols, a locally relevant direction is not enough: a proposed mechanism must match retraining at the behavioral endpoint of its claimed role, and the perturbation actually evaluated must be the one the mathematics describes.
\section{Related Work}
\label{sec:related}
\paragraph{Exact retraining and certified removal.}
Early work framed machine unlearning as the efficient deletion of data-dependent state \citep{cao2015towards,ginart2019making}. SISA limits the training influence of each sample through sharding, slicing, and aggregation, which reduces the amount of retraining required after deletion \citep{bourtoule2021machine}. Certified removal requires a distributional relation between the released model and a model trained without the deleted data \citep{guo2020certified}. Certified removal itself applies a Newton step, an inverse-Hessian influence update in parameter space, to convex models; subsequent work develops deletion guarantees through stability and optimization where the learning problem is tractable \citep{sekhari2021remember,neel2021descent,ullah2021stability}. Our study concerns post-hoc approximate unlearning for deep classifiers. We use exact retraining as the empirical reference, not as a guarantee attained by the evaluated methods.
\paragraph{Approximate unlearning for deep networks.}
The simplest approximate methods fine-tune on the retained set or ascend the forget loss \citep{graves2021amnesiac,thudi2022unrolling}; we use both as controls. Two more structured families modify forget-sensitive parameters or adopt teacher-student objectives. Weight scrubbing, SSD, and SalUn suppress or restrict updates to parameters associated with the forget set \citep{golatkar2020eternal,foster2024ssd,fan2024salun}, whereas Bad Teacher and SCRUB preserve retained knowledge while encouraging disagreement on forgotten samples \citep{chundawat2023bad,kurmanji2023unbounded}. Recent benchmarks and competition analyses show that method rankings are sensitive to datasets, forget sets, initializations, metrics, and baseline tuning \citep{zhao2024hard,triantafillou2024progress,cadet2025deep}, and that weak evaluations can produce a false sense of privacy \citep{hayes2024inexact}, motivating our use of fixed checkpoints, exact forget indices, matched evaluation, and simple controls.
\paragraph{Privacy evaluation.}
A low forget-set accuracy is neither necessary nor sufficient for random instance deletion. Membership inference instead tests whether the model distinguishes training members from nonmembers \citep{shokri2017membership}. Likelihood-ratio attacks provide competitive membership audits while reducing the number of required reference models~\citep{carlini2022membership,zarifzadeh2024rmia}. Recent work formalizes unlearning evaluation as a cryptographic game \citep{tu2025reliable} and audits per-sample efficacy with likelihood inference \citep{naderloui2025rectifying}; these measure aspects of unlearning that the statistics used here do not. We report the SVC-based confidence statistic from the AMUN codebase and RMIA FT-AUC when available, and interpret both relative to the exact retrained model: a numerically smaller value is not better if retraining has a different value.
\paragraph{Perturbation learning and adversarial unlearning.}
Influence functions approximate the effect of upweighting or removing training points through inverse-Hessian vector products \citep{koh2017influence}. In deep nonconvex networks their estimates can be fragile \citep{basu2021fragile}, approximate a proximal Bregman response rather than the retrained solution \citep{bae2022influence}, and depend on the iHVP solver \citep{klochkov2024ihvp}. Influence-based unlearning has applied these updates directly in parameter space \citep{guo2020certified,izzo2021approximate,warnecke2023features}. DPL instead uses related second-order information to learn image and label perturbations that improve supervised training \citep{song2023dpl}; its original objective is neither an unlearning objective nor adversarial minimax training, and to our knowledge it has not previously been evaluated for deletion. Input-space unlearning methods do exist: error-maximizing noise \citep{tarun2023fast} and boundary unlearning \citep{chen2023boundary} perturb inputs or decision boundaries to induce forgetting, and adversarial-example unlearning takes a related route. AMUN constructs nearby adversarial examples for the forget set and updates the model toward deliberately incorrect targets on those examples, thereby supplying explicit deletion pressure near the local decision boundary \citep{ebrahimpour2025amun}. Our work does not propose a stronger adversarial unlearning method. It asks whether an influence-derived direction is useful as a standalone update, as an auxiliary utility term, or as an initialization for adversarial unlearning.
\paragraph{Implementation-sensitive evaluation.}
Adversarial-robustness evaluation has documented how coordinate, normalization, and precision choices silently change the perturbation actually tested \citep{carlini2019evaluating,croce2020reliable}, and unlearning work has argued that deletion claims must be auditable at the level of the implemented algorithm \citep{thudi2022necessity}. Perturbation-based unlearning inherits both concerns: image coordinates, normalization units, projection domains, and numerical precision each change the realized update. We therefore evaluate the implemented perturbation through its end-to-end unlearning outcomes rather than agreement with an intermediate mathematical direction.
\section{Problem Setup and Three Perturbation Roles}
\label{sec:setup}
Let the original training set be $\cD=\cR\cup\cF$, where $\cF$ is the forget set and $\cR$ is the retained set. Write $\rho=|\cF|/|\cD|$. For a parameter vector $\theta$ and loss $\ell$, define the average empirical risks
\begin{equation}
\Risk_A(\theta)=\frac{1}{|A|}\sum_{z\in A}\ell(\theta;z),
\qquad A\in\{\cD,\cR,\cF\}.
\end{equation}
We abbreviate $\Risk_\cR$ and $\Risk_\cF$ as $\Risk_R$ and $\Risk_F$. The original model is $\thetao$, and $\thetar$ denotes an exact retraining run on $\cR$ under the matched training recipe. An approximate procedure returns $\thetau$.
\begin{definition}[Retrain-centered behavioral discrepancy]
Let $M_1,\ldots,M_K$ be behavioral statistics such as forget, retain, and test accuracy and membership-inference scores. For nonnegative weights $w_k$, define
\begin{equation}
\diff(\theta,\thetar)=\sum_{k=1}^{K}w_k\abs{M_k(\theta)-M_k(\thetar)}.
\label{eq:behavioral-diff}
\end{equation}
The exact implementation in our experiments uses the AMUN codebase's average-difference statistic. We report every component rather than relying on the aggregate alone. Computational cost is recorded separately and enters the role-specific gates. We call $\diff(\theta,\thetar)$ the \emph{average discrepancy} (\avgdiff) and use this name throughout.
\end{definition}
\paragraph{Role I: direct forgetting.}
Let $g_F=\nabla_\theta\Risk_F(\thetao)$ and let $H$ be a damped empirical Hessian. The tested influence vector is
\begin{equation}
 v=H^{-1}g_F.
\end{equation}
For a candidate example with input $x$, DPL forms a mixed derivative of the candidate gradient with the sign-reversed influence vector,
\begin{equation}
 \delta_x=\nabla_x\left\langle \nabla_\theta \ell(\thetao;x,y),-v\right\rangle.
\label{eq:dpl-direction}
\end{equation}
The corrected implementation computes this derivative in deterministic raw-image coordinates, applies a normalized or signed version of $\delta_x$ to the same raw image, and then performs a short post-hoc update. This role succeeds only if the resulting model has lower \avgdiff to $\thetar$ than simple post-hoc controls.
\paragraph{Role II: utility regularization.}
Here an explicit adversarial-example objective supplies the forgetting pressure. A utility term is added with weight $\lambda_u$:
\begin{align}
\Risk_{\mathrm{clean}}(\theta)&=\Risk_{\mathrm{explicit}}(\theta)+\lambda_u\Risk_R(\theta),\\
\Risk_{\mathrm{DPL}}(\theta)&=\Risk_{\mathrm{explicit}}(\theta)+\lambda_u\Risk_R^{\mathrm{DPL}}(\theta).
\end{align}
The second loss replaces clean retained examples with retained examples perturbed by a retain-derived influence direction. The matched control is ordinary clean retained-data cross-entropy. The DPL term succeeds only if it improves utility without weakening forgetting or privacy matching.
\paragraph{Role III: adversarial warm start.}
The adversarial unlearning pipeline generates a perturbed forget set by maximizing an attack objective under its repository-defined constraint. We compare official initialization, FGSM initialization, and corrected normalized DPL initialization at one, three, and five refinement steps. The DPL direction cost is included in total runtime. A warm start succeeds only if it reaches comparable unlearning quality faster than the official initialization and improves meaningfully over FGSM at the same step count.
\paragraph{Gate principle.}
Each role is judged against a simpler control chosen before inspecting its final comparison. We do not reinterpret a small gain on one metric as success when another required metric degrades. The concrete gates are given in \cref{sec:protocol} and the full decision tables appear in \cref{app:gates}.
\section{Theoretical Analysis}
\label{sec:analysis}
This section establishes what an influence-derived data direction can and cannot justify: why exact retraining keeps forget accuracy high under random deletion (\cref{prop:random-delete}), the first-order parameter displacement that deletion induces (\cref{prop:retrain-direction}), the input direction that best matches it in one step (\cref{prop:data-direction}), how coordinate errors break that match (\cref{prop:coord-error}), and why parameter alignment is only a surrogate for behavioral agreement (\cref{lem:lipschitz}). All results are local expansions around $\thetao$ under explicit regularity conditions; none is a global unlearning guarantee for a nonconvex network.
\subsection{Random deletion should resemble population prediction}
\begin{proposition}[Expected forget risk after random deletion]
\label{prop:random-delete}
Let $Z_1,\dots,Z_n$ be i.i.d.\ from a population $P$, and suppose
$\mathbb{E}_{Z\sim P}[|\ell(\theta;Z)|]<\infty$ for all $\theta$ under
consideration. Select a retained index set $I_R$ uniformly at random, independently of the sample values, and let $I_F$ be its complement. Let $\hat\theta_R$ be any measurable function of $\{Z_i:i\in I_R\}$. Then
\begin{equation}
\E\!\left[\Risk_F(\hat\theta_R)\mid I_R,\{Z_i:i\in I_R\}\right]
=\E_{Z\sim P}[\ell(\hat\theta_R;Z)].
\end{equation}
The same identity holds for any bounded statistic of the prediction, including accuracy.
\end{proposition}
\begin{proof}
Condition on the retained indices and retained sample values. The forgotten observations remain independent draws from $P$ and are independent of $\hat\theta_R$, because $\hat\theta_R$ is a function only of retained observations. The conditional expectation of each forgotten loss is therefore the population risk of $\hat\theta_R$. Averaging over forgotten indices gives the result. Replacing loss by a bounded prediction statistic gives the second statement.
\end{proof}
\Cref{prop:random-delete} explains why ``forget accuracy down'' is the wrong generic objective for random deletion. The expected forget behavior of exact retraining is population behavior. For random deletion, exact retraining therefore retains high forget accuracy rather than driving it to zero. A method that collapses forget accuracy by damaging the classifier is farther from the deletion target, not closer.
\subsection{First-order parameter displacement from deletion}
\begin{assumption}[Local regularity]
\label{ass:regularity}
The risks $\Risk_R$ and $\Risk_F$ are twice continuously differentiable in a neighborhood of $\thetao$, and the Hessian of $\Risk_R$ is locally Lipschitz there. The original model is stationary for the full empirical risk,
\begin{equation}
(1-\rho)\nabla\Risk_R(\thetao)+\rho\nabla\Risk_F(\thetao)=0.
\end{equation}
The retained Hessian $H_R=\nabla^2\Risk_R(\thetao)$ is invertible on the local subspace under consideration.
\end{assumption}
\begin{proposition}[Local retraining direction]
\label{prop:retrain-direction}
Under \cref{ass:regularity}, let $d=\thetar-\thetao$ and assume $\thetar$ is a stationary point of $\Risk_R$ in the same local basin. Then
\begin{equation}
 d=\frac{\rho}{1-\rho}H_R^{-1}\nabla\Risk_F(\thetao)+O(\norm{d}^2).
\label{eq:retrain-direction}
\end{equation}
\end{proposition}
\begin{proof}
Taylor expansion of the retained gradient around $\thetao$ gives
\begin{equation}
0=\nabla\Risk_R(\thetar)=\nabla\Risk_R(\thetao)+H_Rd+O(\norm{d}^2).
\end{equation}
Stationarity of the full risk implies
$\nabla\Risk_R(\thetao)=-\rho(1-\rho)^{-1}\nabla\Risk_F(\thetao)$.
Multiplying by $H_R^{-1}$ and rearranging yields \cref{eq:retrain-direction}.
\end{proof}
In a network with a singular or indefinite Hessian, the damped inverse used in practice defines a regularized surrogate direction. It is not algebraically identical to the displacement of an exact retraining solution.
The sign in \cref{eq:retrain-direction} is sometimes counterintuitive. Removing $\cF$ does not correspond to stepping along negative forget gradient. At a stationary full-data solution, it corresponds locally to an inverse-retained-Hessian transform of the positive forget gradient. Gradient ascent on the forget loss can still be used as a heuristic, but it is not the first-order retraining displacement.
\subsection{The input direction that matches a desired update}
Consider one model update on a candidate set $C$ whose inputs are parameterized by a perturbation vector $\delta$:
\begin{equation}
\theta^+(\delta)=\thetao-\eta\nabla_\theta\Risk_C(\thetao;\delta).
\end{equation}
Let $\Delta(\delta)=\theta^+(\delta)-\thetao$, $\Delta_0=\Delta(0)$, and define the local parameter-matching objective
\begin{equation}
J(\delta)=\frac12\norm{\Delta(\delta)-d}^2,
\label{eq:matching-objective}
\end{equation}
where $d$ is the desired deletion displacement from \cref{prop:retrain-direction}. Let
$B=\nabla^2_{\theta\delta}\Risk_C(\thetao;0)$.
\begin{proposition}[Steepest first-order data direction]
\label{prop:data-direction}
Assume that $\nabla_\theta\Risk_C(\thetao;\delta)$ is differentiable in $\delta$ near zero and that its Jacobian is locally Lipschitz. Then
\begin{equation}
\Delta(\delta)=\Delta_0-\eta B\delta+O(\eta\norm{\delta}^2).
\end{equation}
For a small Euclidean budget $\norm{\delta}\leq\epsilon$, the direction that decreases \cref{eq:matching-objective} most rapidly at $\delta=0$ is
\begin{equation}
\delta^*=\epsilon\frac{B^\top(\Delta_0-d)}{\norm{B^\top(\Delta_0-d)}}
\label{eq:steepest-data}
\end{equation}
when the denominator is nonzero. If the unperturbed candidate update $\Delta_0$ is neglected and $d$ is replaced by the influence approximation in \cref{eq:retrain-direction}, then
\begin{equation}
\delta^*\ \propto\ -B^\top H_R^{-1}\nabla\Risk_F(\thetao).
\label{eq:dpl-local}
\end{equation}
\end{proposition}
\begin{proof}
Taylor expansion of the candidate gradient gives
$\nabla_\theta\Risk_C(\thetao;\delta)=\nabla_\theta\Risk_C(\thetao;0)+B\delta+O(\norm{\delta}^2)$.
Multiplying by $-\eta$ gives the expansion of $\Delta$. Differentiating \cref{eq:matching-objective} at zero yields
\begin{equation}
\nabla_\delta J(0)=-\eta B^\top(\Delta_0-d).
\end{equation}
The minimizer of the first-order approximation $J(0)+\langle\nabla J(0),\delta\rangle$ over the Euclidean ball points opposite the gradient, which gives \cref{eq:steepest-data}. Setting $\Delta_0\approx0$ and substituting \cref{eq:retrain-direction} gives \cref{eq:dpl-local}; the positive scalar $\rho/(1-\rho)$ does not affect the normalized direction.
\end{proof}
Eq.~\eqref{eq:dpl-local} provides a local justification for the DPL direction. Under the stated approximations, it is a first-order input direction for matching one parameter update. This result does not extend to signed perturbations, repeated nonlinear updates, accuracy, or membership-inference behavior.
\subsection{Coordinate errors can invalidate the local guarantee}
Let $s=B^\top(\Delta_0-d)$ be the ideal direction in \cref{eq:steepest-data}. Suppose an implementation computes $\tilde s=s+e$, where $e$ captures crop, flip, normalization, indexing, or numerical errors.
\begin{proposition}[First-order tolerance to direction error]
\label{prop:coord-error}
Assume $\tilde s\neq0$ and that $\nabla J$ is locally Lipschitz. For the normalized implemented step $\delta=\epsilon\tilde s/\norm{\tilde s}$, the first-order change in $J$ is
\begin{equation}
J(\delta)-J(0)=-\eta\epsilon\frac{\langle s,\tilde s\rangle}{\norm{\tilde s}}+O(\epsilon^2).
\end{equation}
Thus the intended first-order decrease occurs exactly when $\langle s,\tilde s\rangle>0$. A sufficient condition is $\norm{e}<\norm{s}$.
\end{proposition}
\begin{proof}
From the proof of \cref{prop:data-direction}, $\nabla J(0)=-\eta s$. Taking the directional Taylor expansion along $\delta$ gives the displayed expression. The decrease condition follows from the sign of the linear term. Finally,
\begin{equation}
\langle s,s+e\rangle=\norm{s}^2+\langle s,e\rangle
\geq \norm{s}^2-\norm{s}\norm{e}>0
\end{equation}
when $\norm{e}<\norm{s}$.
\end{proof}
A random crop or flip is not a small additive error in raw pixel coordinates. It can permute the coordinates on which $B^\top(\Delta_0-d)$ is defined. Likewise, a derivative in normalized units cannot be interpreted as a raw-space perturbation without the normalization Jacobian. These errors can violate the inner-product condition even when tensor shapes agree.
\subsection{Parameter alignment is only a surrogate}
\begin{lemma}[Lipschitz behavioral upper bound]
\label{lem:lipschitz}
Suppose each scalar metric $M_k$ is $L_k$-Lipschitz in a neighborhood of $\thetar$. Then the discrepancy in \cref{eq:behavioral-diff} obeys
\begin{equation}
\diff(\theta,\thetar)\leq\left(\sum_{k=1}^{K}w_kL_k\right)\norm{\theta-\thetar}.
\end{equation}
\end{lemma}
\begin{proof}
Apply the Lipschitz inequality $|M_k(\theta)-M_k(\thetar)|\leq L_k\norm{\theta-\thetar}$ to each term and sum with weights $w_k$.
\end{proof}
Accuracy and thresholded membership decisions are not globally Lipschitz. RMIA and SVC statistics also depend on reference distributions and nonlinear scores. Therefore \cref{lem:lipschitz} cannot turn local parameter alignment into a deletion guarantee. It instead explains why end-to-end behavioral evaluation remains necessary even when the influence calculation is mathematically correct.
\section{Implementation and Numerical Audit}
\label{sec:audit}
The CIFAR-10 audit was performed before the final role comparisons, and the same tensor-level procedure was applied during the Tiny ImageNet external check. The direct path reuses the public DPL implementation~\citep{song2023dpl} for its dataset class, Hessian solve, mixed-derivative routine, and perturbation write-back, wrapped in our deletion-specific guide and candidate selection; the adversarial pipeline is the public AMUN repository~\citep{ebrahimpour2025amun}. We traced the actual tensors, transforms, signs, projection steps, and output files rather than relying on code comments, and we state for each finding which codebase it originates in.
\paragraph{Spatial and unit mismatch in the original direct path.}
The dataset class used for direction construction exposes writable image tensors (we call it the mutable variant). It stores raw images in $[0,\ 1]$. During direction construction, however, its data loader applied random crop, random horizontal flip, and CIFAR normalization, as in the released DPL training loader. The mixed derivative was therefore defined on a randomly transformed normalized tensor. The application path then added $\alpha\,\mathrm{sign}(\delta_x)$ to the unaugmented raw image. Tensor shape was preserved, but spatial coordinates and numerical units were not. By \cref{prop:coord-error}, the intended local improvement is not preserved unless the implemented and ideal directions retain a positive inner product.
\paragraph{Corrected raw-space mode.}
We introduced an explicit deterministic mode that leaves existing defaults unchanged. It disables stochastic spatial transforms for direction construction, differentiates with respect to the exact raw $[0,1]$ tensor, applies normalization as a differentiable operation immediately before the model forward pass, adds the direction to the same raw tensor, and clamps only in raw-image space. Direction validation uses this mode throughout.
\paragraph{Coordinate-valid alignment diagnostic.}
On the fixed 1,024-example diagnostic subset, the corrected normalized and signed directions, which share raw-image coordinates, have mean per-sample cosine similarity 0.653; 99.7\% of samples have cosine above 0.5. The two corrected representations therefore satisfy the positive-inner-product condition of \cref{prop:coord-error} with respect to each other. A legacy-to-corrected cosine cannot be computed: the original stochastic path did not persist the per-sample crop offsets and flip decisions needed to pull the legacy direction back to raw coordinates, so we report that comparison as undefined rather than align tensors from incompatible coordinate systems.
\paragraph{Sign verification.}
The implementation computes $v=H^{-1}g_F$ and passes $-v$ to the mixed-derivative routine. The saved image and label directions therefore already contain one unlearning sign reversal. The write-back path adds the saved direction and does not reverse it again. We nevertheless evaluate both signs because \cref{prop:data-direction} is local and the realized training procedure includes normalization, signed directions, projection, and multiple updates.
\paragraph{Label perturbation is numerically inactive.}
For 5,000 forget candidates, the label direction had $\ell_\infty$ norm $3.27\times10^{-5}$ and $\ell_2$ norm $5.86\times10^{-4}$. The original correct-class soft-target probability was 0.999591708. At $\beta=0.007$, both direction signs produced zero measurable target-space change, zero KL divergence, and zero argmax changes in float32. Even $4\beta$ changed a target coordinate by only about $1.06\times10^{-8}$. We therefore disable label perturbation in the utility and warm-start studies. Any model change in the earlier label-only arm is attributable to the surrounding fine-tuning procedure, not to a changed target.
\paragraph{Adversarial attack coordinates.}
The matched AMUN repository initializes attacks from a normalized clean tensor and applies an internal $[0,1]$ clamp. This mixes normalized coordinates with raw-image bounds. We retain that behavior in the warm-start comparison so that initialization is the only changed factor, and we verify bitwise regression to the default path. We do not present the resulting perturbations as a mathematically strict raw-image $\ell_2$ threat set. This issue is separated from the DPL conclusion because it affects FGSM and DPL initializations alike.
\paragraph{A distinct Tiny ImageNet preprocessing defect.}
The second-dataset check exposed a different inconsistency. The VGG-19 source model was trained and validated with channel normalization, whereas the official Tiny ImageNet adversarial and downstream evaluation paths used raw tensors for model forward. The source checkpoint reached 57.59\% under its own normalized validation transform but only 12.28\% under the common raw evaluator. Inserting differentiable normalization before the attack forward, while retaining raw-space projection and clipping, changed full-AMUN forget, retain, and test accuracy by 57.42, 50.25, and 28.95 percentage points. Preprocessing is part of the evaluated predictor, $h_\theta(x)=f_\theta(T(x))$. A shared evaluator is valid only when $T$ is compatible with the transform used to train each checkpoint.
\paragraph{Interpretation of the audit.}
The corrected CIFAR-10 direct-role experiments rule out the original coordinate mismatch as an explanation for the seed-1 direct result. Conversely, a gain obtained under the mismatched path would not isolate the effect of the influence-derived direction. The audit determines which perturbation the experiment actually evaluates.
\section{Experiments}
\label{sec:experiments}

\subsection{Protocol and baselines}
\label{sec:protocol}

The detailed role analyses use CIFAR-10, ResNet-18, source seed 1, forget-index seed 1, and 5,000 randomly forgotten examples. The exact retrained model uses the same retained split and training recipe. Direction construction uses 5,000 guide examples, 5,000 candidates, a damped conjugate-gradient solve with at most 20 iterations, and deterministic seed 1729. The matched evaluation reports forget, retain, and test accuracy, the SVC membership-confidence statistic, RMIA FT-AUC when available, the codebase's average discrepancy to exact retraining, and wall-clock cost. We then repeat the frozen final role configurations for paired source/forget seeds $(1,1)$, $(10,10)$, and $(100,100)$, each against its exact-seed retraining reference, without retuning direction signs, scales, CG budget, utility weight, or attack-step count. We additionally report a one-seed Tiny ImageNet and VGG-19 check with 10,000 randomly forgotten examples. It is an external-validity and implementation audit, not a second decisive gate.

The wider baseline screening includes three original-model seeds and three forget-index seeds for FT, gradient ascent (GA), SalUn, and AMUN. Retraining is available for the three source seeds. Because these counts differ, the aggregate is descriptive rather than an inferential comparison. The seed-1 role analyses use identical assets, while the final cross-seed validation uses matched assets and exact retraining references within each paired seed.

The key baselines are chosen by role. Direct DPL is compared with exact retraining, SalUn, and AMUN. The utility plug-in is compared with the same explicit forgetting objective plus ordinary clean retain CE. The warm start is compared with official initialization and FGSM at matched refinement steps. GA is retained as a diagnostic baseline because it demonstrates that a low forget accuracy can be obtained by destroying the entire classifier.

\begin{table}[t]
\caption{Matched baseline screening on CIFAR-10 and ResNet-18. Accuracies are percentages. The method count differs because three source seeds and three forget-index seeds were available for approximate methods, while retraining was aggregated over the three source seeds. Lower average discrepancy is closer to retraining.}
\label{tab:baseline-screen}
\centering
\small
\begin{tabular}{lrrrrrr}
\toprule
Method & $n$ & Forget & Retain & Test & SVC conf. & Avg. diff. \\
\midrule
Retrain & 3 & 94.75 & 100.00 & 94.35 & 12.62 & 0.13 \\
FT & 9 & 99.74 & 99.84 & 93.82 & 2.53 & 3.98 \\
GA & 9 & 20.02 & 20.35 & 20.39 & 36.96 & 64.79 \\
SalUn & 9 & 99.10 & 99.93 & 94.16 & 8.19 & 2.30 \\
AMUN & 9 & 99.70 & 99.80 & 93.77 & 2.83 & 3.92 \\
\bottomrule
\end{tabular}

\end{table}

Table~\ref{tab:baseline-screen} illustrates the evaluation logic. Retraining has high forget accuracy because random forgotten instances remain predictable from the retained distribution. GA lowers forget accuracy to 20.02\%, but retain and test accuracy collapse to about 20\%. SalUn is the closest approximate baseline in this screening. FT and AMUN preserve high accuracy but remain farther from the retrained privacy and behavior statistics.

\subsection{Role I: direct forgetting does not approach retraining}

We first increased the original lightweight DPL configuration from a small perturbed subset to 5,000 guides and 5,000 candidates, used 20 conjugate-gradient (CG) iterations, separated image and label perturbations, and recomputed the direction between update stages. Across ten saved diagnostic checkpoints, forget accuracy remained 100\% and forget-retain accuracy separation was nonpositive. We then corrected the coordinate mismatch and tested both direction signs, signed and normalized representations, and $\alpha\in\{1,2,4,8\}/255$.

\begin{table}[t]
\caption{Single-seed comparison for direct unlearning. Best DPL direct is selected optimistically as the diagnostic checkpoint with the smallest average discrepancy. Even this selected checkpoint is farther from retraining than SalUn and AMUN.}
\label{tab:direct}
\centering
\small
\resizebox{\textwidth}{!}{\begin{tabular}{lrrrrrr}
\toprule
Method & Forget & Retain & Test & SVC conf. & Avg. diff. & RMIA FT-AUC \\
\midrule
Retrain & 94.58 & 100.000 & 94.39 & 12.38 & 0.000 & 0.4958 \\
SalUn & 99.04 & 99.940 & 94.27 & 10.04 & 1.745 & 0.5065 \\
AMUN & 99.72 & 99.780 & 93.47 & 2.92 & 3.935 & 0.5303 \\
\textbf{Best DPL direct} & 100.00 & 99.998 & 94.53 & 0.48 & 4.366 & 0.5414 \\
\bottomrule
\end{tabular}
}
\end{table}

The best direct checkpoint preserves test accuracy at 94.53\%, slightly above the retrained value. It nevertheless retains 100\% forget accuracy, has SVC confidence 0.48 rather than the retrained 12.38, and has RMIA FT-AUC 0.5414 rather than 0.4958. Its average discrepancy is 4.366, compared with 1.745 for SalUn and 3.935 for AMUN. High utility is therefore not evidence of deletion matching.

\begin{figure}[t]
    \centering
    \begin{subfigure}[t]{0.45\linewidth}
        \centering
        \includegraphics[width=\linewidth]{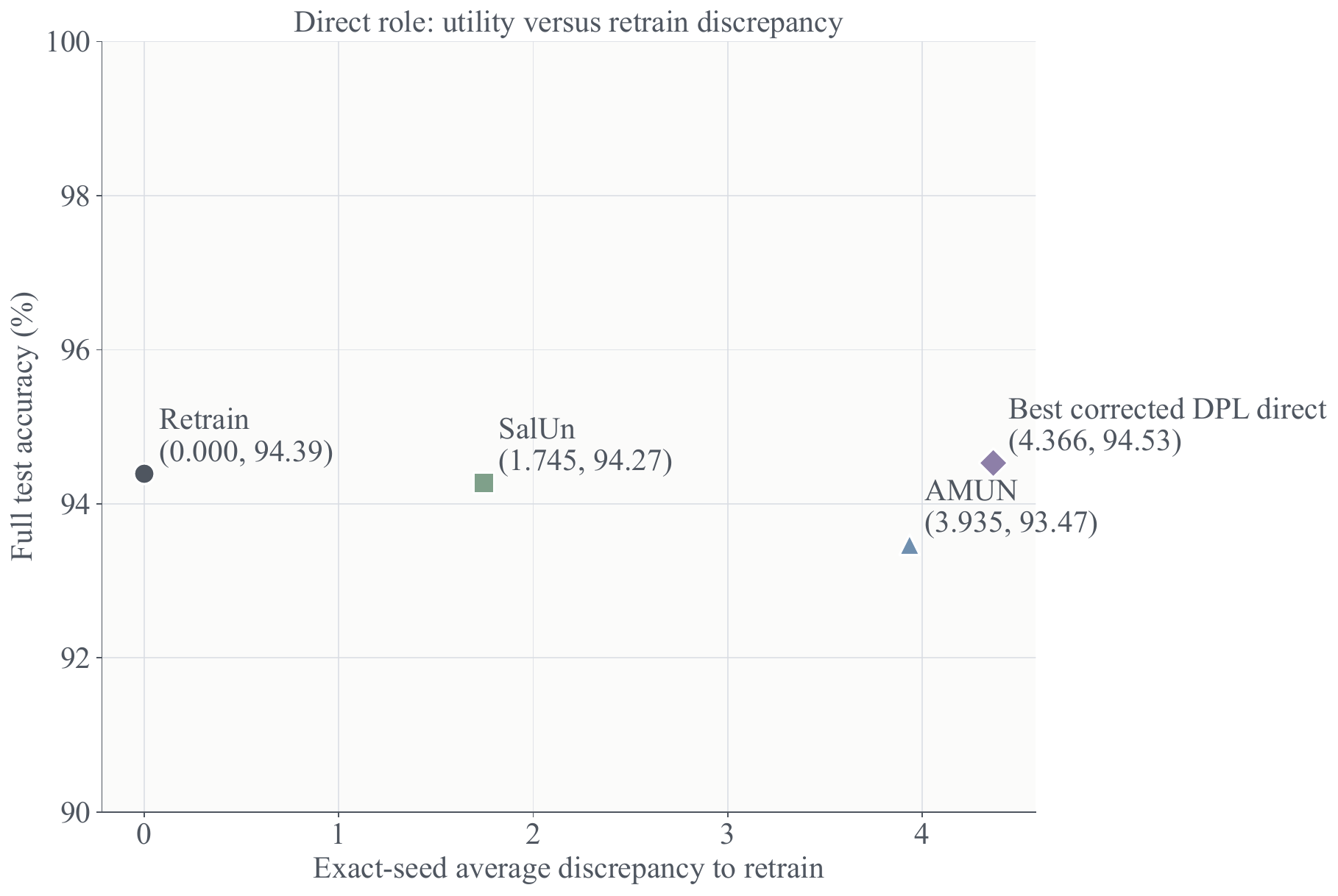}
        \caption{The best DPL direct checkpoint preserves test utility but remains far from exact retraining}
        \label{fig:direct-a}
    \end{subfigure}
    \hfill
    \begin{subfigure}[t]{0.45\linewidth}
        \centering
        \includegraphics[width=\linewidth]{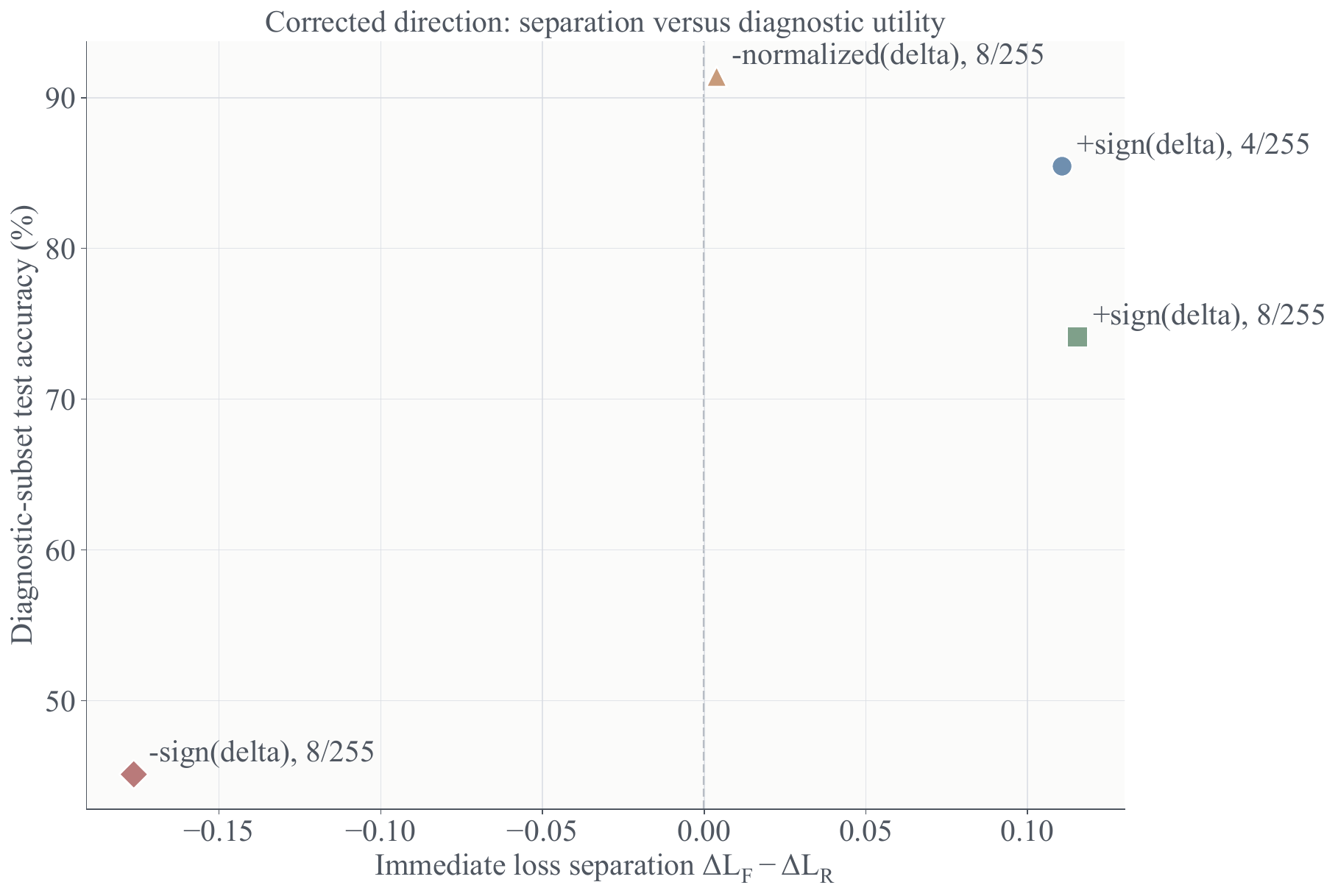}
        \caption{Corrected directions can produce positive immediate loss separation, but strong effects coincide with large utility loss, while the utility-preserving effect is very small.}
        \label{fig:direct-b}
    \end{subfigure}
    \caption{Direct-role comparison. The best utility-preserving DPL checkpoint remains far from exact retraining, while larger immediate loss separations coincide with utility loss.}
    \label{fig:direct}
\end{figure}

The corrected direction can increase forget loss relative to retain loss, but the larger separations coincide with substantial test-accuracy loss. The utility-preserving configuration yields a separation of only 0.0039. After a short update it changes forget accuracy only 0.098 percentage points more than retain accuracy, equivalent to one additional changed example in the 1,024-sample diagnostic subset. A clean fine-tuning control performs at least as well. For seed 1, we therefore reject the direct-role gate. The frozen three-seed validation in \cref{sec:multiseed} reaches the same gate-level decision for all three paired seeds.

\subsection{Role II: DPL utility does not reliably beat clean retain CE}

All utility arms start from the same adversarial-unlearning checkpoint, use identical data order, SGD settings, 40 updates, and frozen batch-normalization statistics. Each update retains the same explicit forgetting objective. The only difference is no extra utility loss, $0.1$ times clean retain CE, or $0.1$ times CE on DPL-perturbed retained examples. The corrected utility direction is normalized, uses raw coordinates, and applies an effective perturbation with measured $\ell_\infty$ magnitude 0.00273. Direction construction uses all 5,000 guide and candidate examples and takes 58.29 seconds.

\begin{table}[t]
\caption{Utility-role comparison under matched forgetting. DPL utility and clean retain CE differ by only 0.02 percentage points in forget accuracy. Lower confidence gap is closer to retraining.}
\label{tab:utility}
\centering
\small
\resizebox{\textwidth}{!}{\begin{tabular}{lrrrrrr}
\toprule
Configuration & Forget & Retain & Test & Forget loss & Retain loss & Conf. gap \\
\midrule
Explicit only & 99.62 & 99.6422 & 93.18 & 0.013257 & 0.010861 & 10.44 \\
+ clean retain CE & 99.64 & 99.7089 & 93.26 & 0.012321 & 0.009343 & 10.52 \\
+ DPL utility & 99.66 & 99.6911 & 93.29 & 0.012457 & 0.009649 & 10.42 \\
\bottomrule
\end{tabular}
}
\end{table}

\begin{figure}[t]
    \centering
    \begin{subfigure}[t]{0.45\linewidth}
        \centering
        \includegraphics[width=\linewidth]{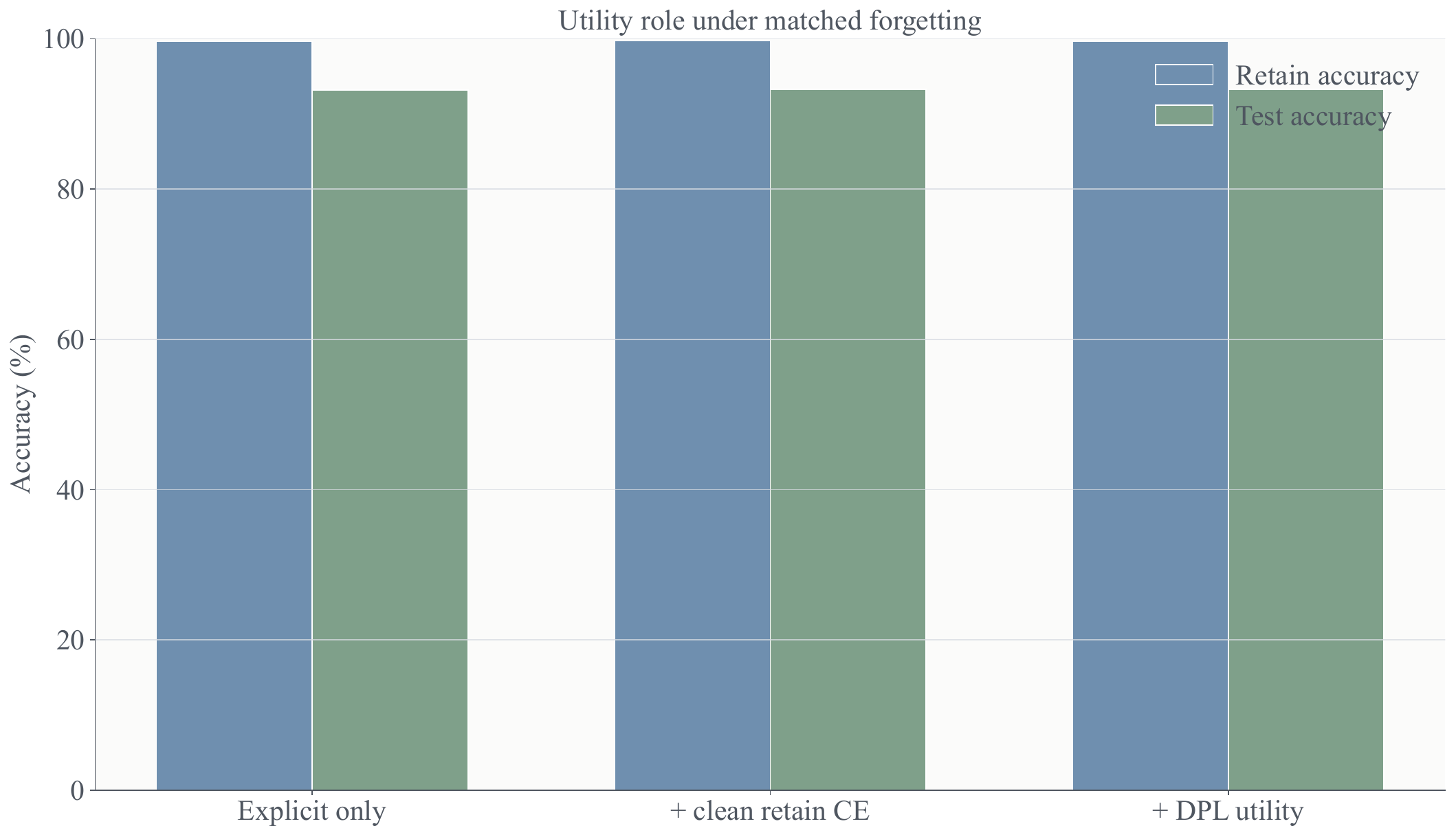}
        \caption{Unlearning performance.}
        \label{fig:utility-a}
    \end{subfigure}
    \hfill
    \begin{subfigure}[t]{0.45\linewidth}
        \centering
        \includegraphics[width=\linewidth]{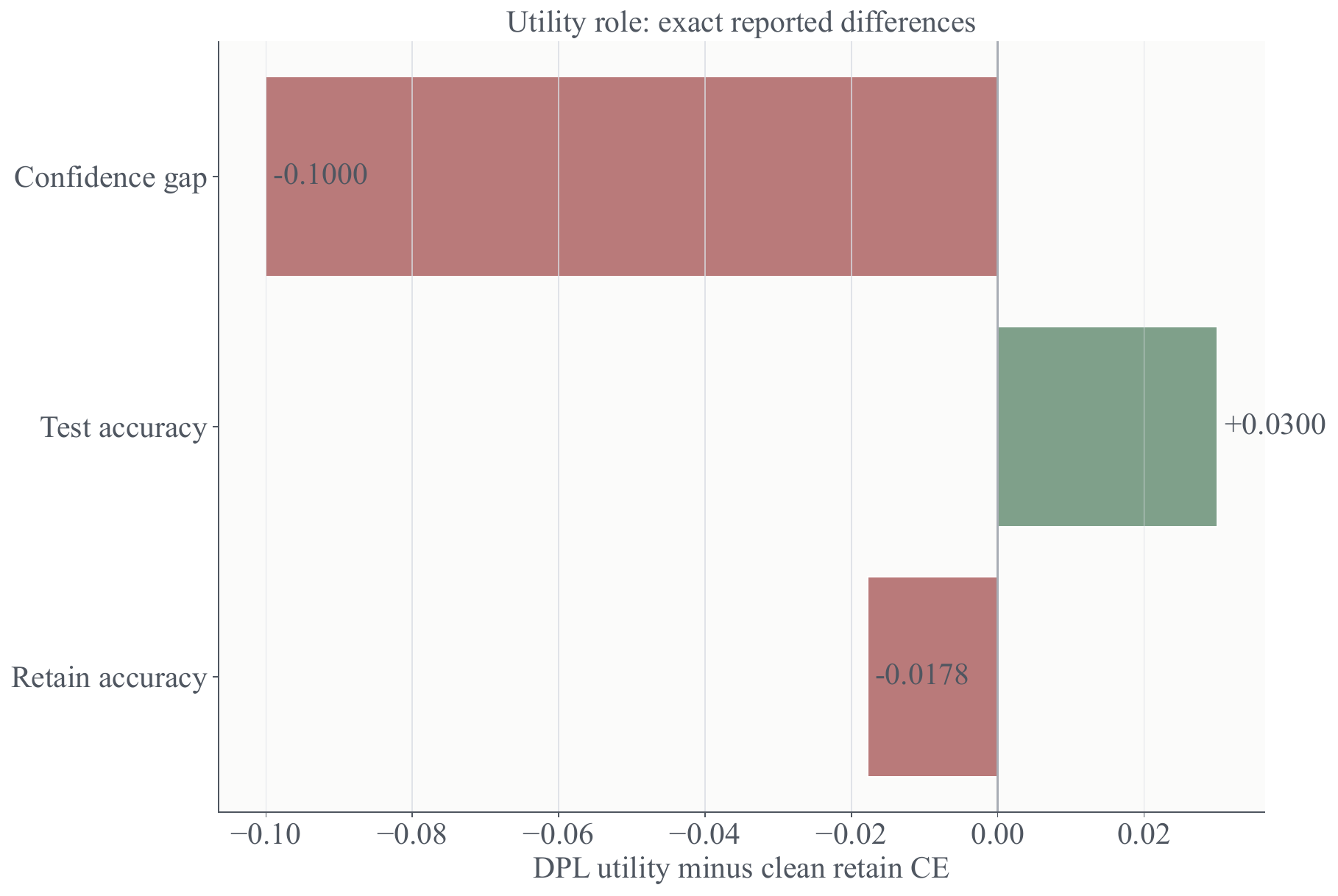}
        \caption{Utility difference.}
        \label{fig:utility-b}
    \end{subfigure}
    \caption{Utility-role evidence. DPL utility increases test accuracy by 0.030 percentage points relative to clean retain CE, but decreases retain accuracy by 0.018 points and increases retain loss.}
    \label{fig:utility}
\end{figure}

The two regularized arms satisfy the predeclared matched-forgetting tolerance of 0.25 percentage points. DPL utility gives 0.03 points higher test accuracy and a 0.10 smaller confidence gap, but 0.0178 points lower retain accuracy and higher retain loss. The gate required retain and test accuracy not to decrease, at least one strict utility improvement, and no worse confidence gap. DPL fails the retain condition. More importantly, its changes are tiny relative to a much simpler regularizer and come with nontrivial direction-computation cost. For seed 1, we reject the utility-role gate. The frozen cross-seed result is mixed: one of three paired seeds passes the utility gate, as reported in \cref{sec:multiseed}.

\subsection{Role III: warm starting adds cost without distinct convergence}

The matched adversarial pipeline uses official clean-tensor initialization and up to 50 steps per dynamic epsilon stage. We compare the official initialization, FGSM, and corrected normalized DPL initialization at one, three, and five steps. The DPL initial raw-space norm is fixed at $2/255$ before repository projection. All arms use the same adversarial labels, objective, projection code, downstream unlearning updates, data order, and evaluation. DPL direction time is included in the total.

\begin{figure}[t]
    \centering
    \begin{subfigure}[t]{0.45\linewidth}
        \centering
        \includegraphics[width=\linewidth]{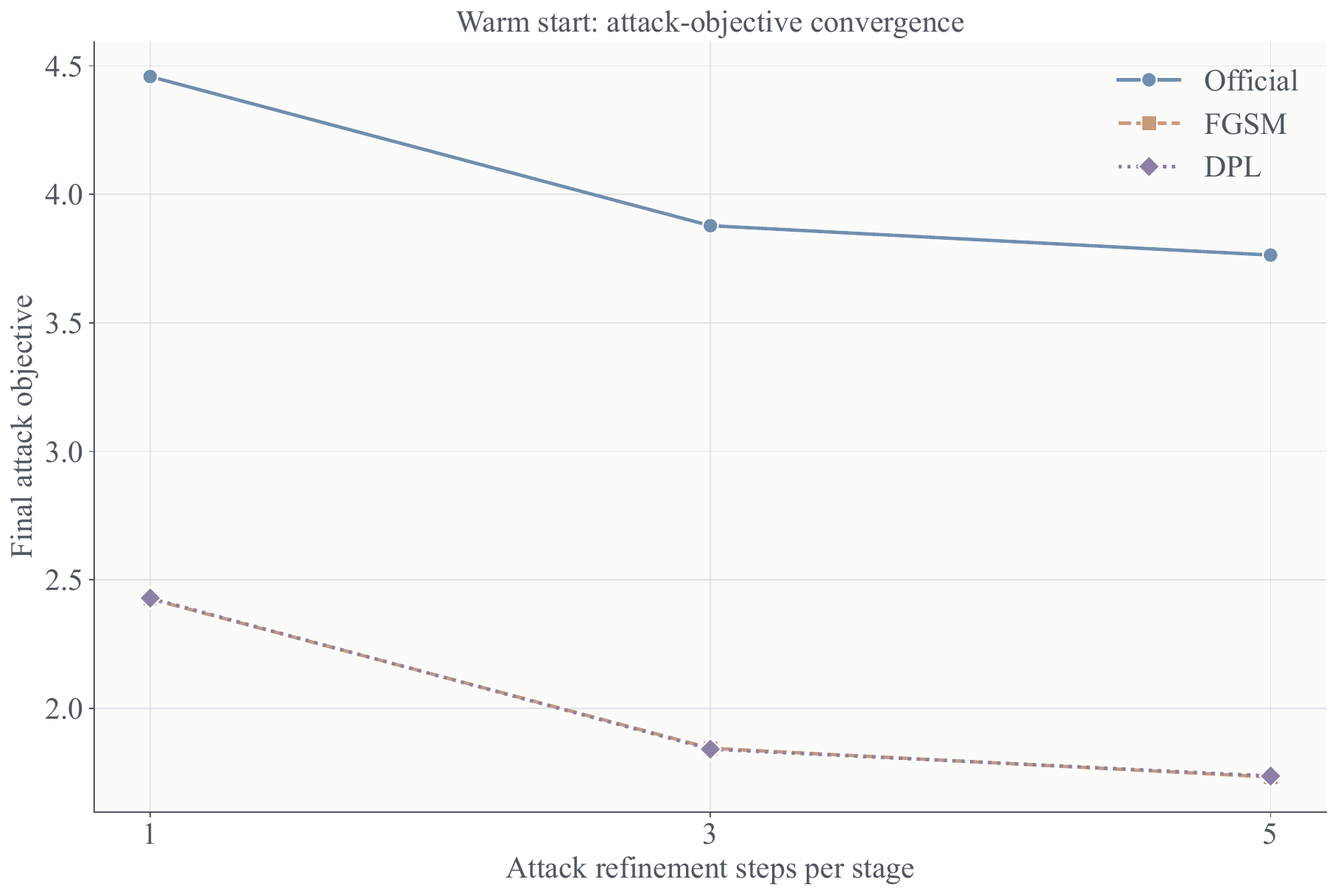}
        \caption{DPL and FGSM have almost identical attack trajectories and success rates. Official initialization reaches the strongest objective and nearly 100\% success.}
        \label{fig:warm-a}
    \end{subfigure}
    \hfill
    \begin{subfigure}[t]{0.45\linewidth}
        \centering
        \includegraphics[width=\linewidth]{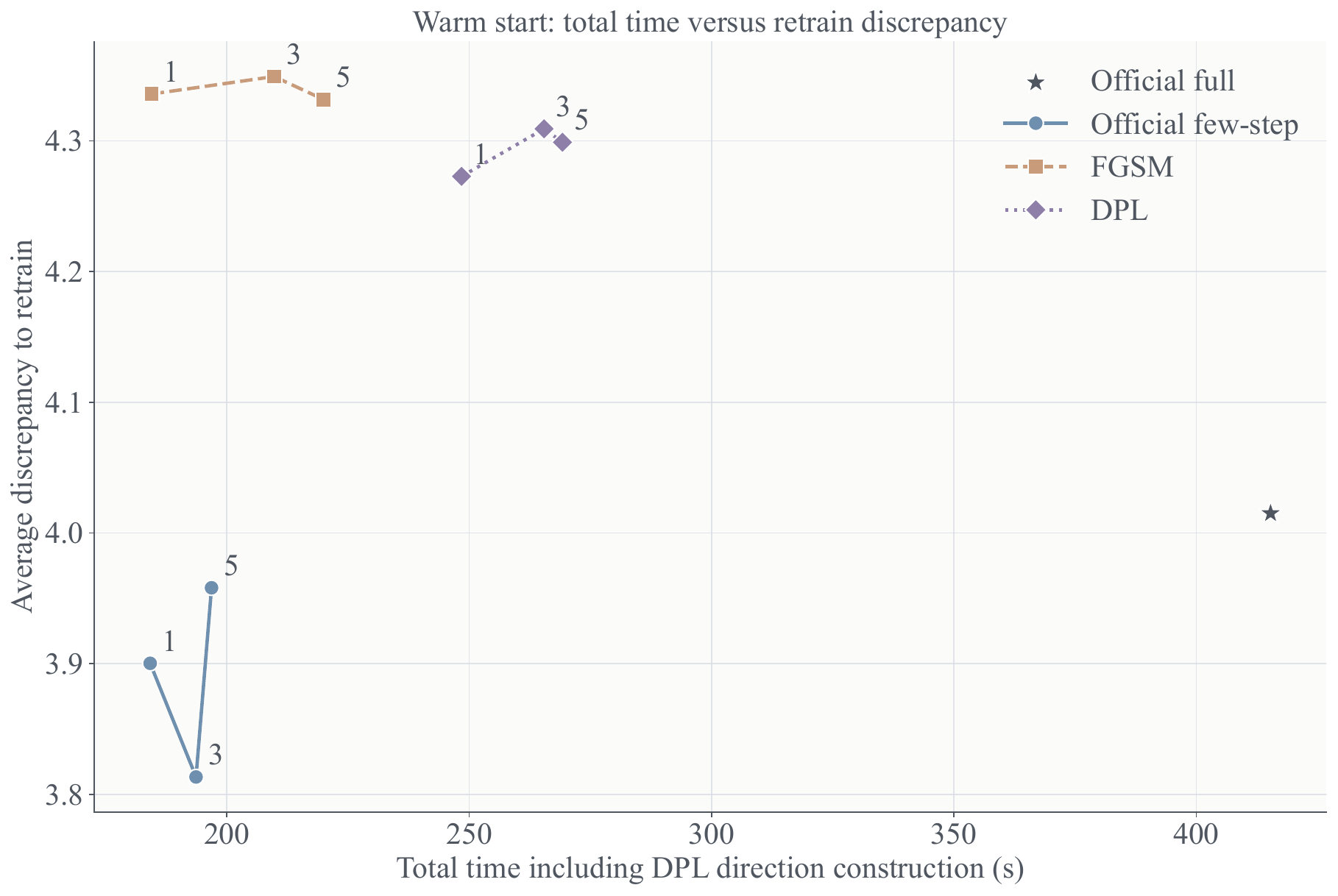}
        \caption{Reducing refinement steps explains the runtime gain. DPL remains slower than the official few-step control after its 65.99-second direction computation is included.}
        \label{fig:warm-b}
    \end{subfigure}
    \caption{Warm-start comparison. DPL and FGSM follow similar trajectories, while DPL remains slower after direction construction.}
    \label{fig:warm}
\end{figure}
\begin{table}[t]
\caption{Warm-start comparison. Success is attack success in percent. Total time includes attack generation, downstream unlearning, and DPL direction construction when applicable.}
\label{tab:warm}
\centering
\scriptsize
\resizebox{\textwidth}{!}{\begin{tabular}{lrrrrrr}
\toprule
Arm & Success & Forget & Test & Avg. diff. & Total time (s) & Backward calls \\
\midrule
Official, full & 100.00 & 99.70 & 93.63 & 4.015 & 415.3 & 25168 \\
Official, 1 step & 100.00 & 99.72 & 93.54 & 3.900 & 184.2 & 4841 \\
Official, 3 steps & 100.00 & 99.64 & 93.55 & 3.813 & 193.6 & 5664 \\
Official, 5 steps & 99.98 & 99.70 & 93.60 & 3.958 & 196.8 & 6505 \\
FGSM, 1 step & 56.98 & 99.90 & 93.64 & 4.336 & 184.5 & 4881 \\
FGSM, 3 steps & 56.98 & 99.90 & 93.70 & 4.349 & 209.8 & 5704 \\
FGSM, 5 steps & 56.96 & 99.90 & 93.57 & 4.332 & 219.9 & 6545 \\
DPL, 1 step & 56.88 & 99.82 & 93.91 & 4.273 & 248.4 & 8517 \\
DPL, 3 steps & 56.88 & 99.94 & 93.86 & 4.309 & 265.4 & 9371 \\
DPL, 5 steps & 56.86 & 99.88 & 93.95 & 4.299 & 269.2 & 10216 \\
\bottomrule
\end{tabular}
}
\end{table}

DPL and FGSM begin with attack objectives 2.0928 and 2.0869 and maintain nearly identical 56.9\% success across one, three, and five steps. Official initialization reaches about 100\% success and a larger final objective. DPL produces slightly higher test accuracy than FGSM in some arms, but its average discrepancy and RMIA remain worse than the official few-step controls. No DPL arm satisfies the full matched-quality gate because average discrepancy and membership gap exceed the allowed tolerance.

Full official AMUN takes 415.3 seconds. Its one-step and three-step variants take 184.2 and 193.6 seconds, respectively. DPL one-step takes 248.4 seconds, including 65.99 seconds for direction construction. Thus, the runtime reduction comes from using fewer refinement steps, while DPL adds direction-construction time.

\paragraph{Strict raw-space check.}
Because the repository-matched sweep above inherits its normalization and clamp convention, the final validation repeats the three-step comparison with adversarial images represented, projected, and clipped in raw $[0,1]$ space, with normalization applied only before model forward. Across paired seeds 1, 10, and 100, DPL never passes the full warm-start gate against both official and FGSM initialization. Its direction construction adds 47.2--63.0 seconds relative to FGSM, and its total time is 29.5--35.5\% higher than the faster simple initialization. The paired retrain-discrepancy results appear in \cref{tab:multiseed}; full raw-space metrics are in \cref{tab:rawspace-warm-app}.

\subsection{Cross-seed validation with frozen configurations}
\label{sec:multiseed}

We next repeat the final role configurations for paired source/forget seeds $(1,1)$, $(10,10)$, and $(100,100)$. No DPL direction sign, perturbation scale, CG budget, utility weight, attack-step count, or final-validation decision threshold is changed across the three-seed matrix, and each run is compared with its exact-seed retraining reference. \Cref{tab:multiseed} reports the paired change in \avgdiff (DPL minus its matched control); negative values favor DPL on this aggregate, but the gate outcome uses all fixed role criteria rather than \avgdiff alone.

\begin{table}[t]
\caption{Frozen three-seed CIFAR-10 validation. Entries are paired changes in \avgdiff (DPL minus the matched control), so negative values favor DPL on this aggregate. The mean is reported with the sample standard deviation across the three paired seeds. ``Passes'' counts full role-gate passes, not wins on \avgdiff alone.}
\label{tab:multiseed}
\centering
\scriptsize
\resizebox{\textwidth}{!}{\begin{tabular}{llrrrrl}
\toprule
Role & Matched comparison & Seed 1 & Seed 10 & Seed 100 & Mean $\pm$ SD & Gate outcome \\
\midrule
Direct & DPL $-$ no perturbation & +0.063 & +0.033 & -0.011 & +0.028 $\pm$ 0.037 & No pass (0/3) \\
Utility & DPL $-$ clean retain CE & -0.023 & -0.044 & +0.168 & +0.034 $\pm$ 0.117 & Mixed (1/3) \\
Warm start & DPL $-$ official init & +0.195 & +0.036 & -0.310 & -0.026 $\pm$ 0.258 & No pass (0/3) \\
Warm start & DPL $-$ FGSM & -0.052 & -0.071 & +0.153 & +0.010 $\pm$ 0.124 &  \\
\bottomrule
\end{tabular}
}
\end{table}

Direct DPL does not pass the matched direct-role gate in any seed. Seed 100 gives a small \avgdiff improvement of 0.011, but this is below the fixed practical threshold and is accompanied by 0.087 and 0.33 percentage-point decreases in retain and test accuracy. 

The utility result is mixed across seeds. Seed 10 passes the matched utility gate, whereas seeds 1 and 100 do not. In seed 100, DPL utility raises retain and test accuracy by 1.36 and 1.37 percentage points, respectively, but increases the SVC-confidence gap to retraining by 2.14 points and worsens Avg.~Diff.\ by 0.168. Seed 100 therefore improves conventional utility while moving farther from the retraining reference. Across the three seeds, DPL utility does not reliably improve over clean retain CE.

Under the strict raw-space comparison, no seed passes the warm-start gate against both simple initializations. DPL has lower \avgdiff than FGSM for seeds 1 and 10 but not seed 100, whereas it improves over official initialization only for seed 100. It therefore never beats both simpler initializations under the full gate in any seed, and it remains slower once direction construction is counted. RMIA is available for every arm in this final CIFAR-10 validation; the complete seed-level metrics are included in the supplementary validation records.

\subsection{Tiny ImageNet: bounded external check and preprocessing insight}
\label{sec:tiny}

We next ran a fixed one-seed check on Tiny ImageNet with VGG-19, source seed 1, forget-index seed 1, and random 10\% deletion. The training set contains 100,000 examples, with 90,000 retained and 10,000 forgotten. We reused the final direction conventions selected on CIFAR-10 and did not tune direction signs, perturbation magnitudes, CG iterations, utility weights, attack steps, or architecture-specific DPL hyperparameters. Exact online RMIA could not be run because the official 128-model Tiny ImageNet reference matrix was unavailable. The table therefore reports the ordinary retrain-centered statistics and complete attributable runtime.

\begin{table}[H]
\caption{One-seed Tiny ImageNet external check with VGG-19 and random 10\% deletion. Accuracies are percentages. Lower Avg. Diff is closer to the exact retrained model. Total time includes direction construction when applicable. ``Raw-forward fixed'' keeps raw-space attack projection and clipping but restores the source-training normalization before model forward. The preprocessing rows are implementation diagnostics, not additional role-level evidence.}
\label{tab:tiny}
\centering
\scriptsize
\resizebox{\textwidth}{!}{\begin{tabular}{llrrrrr}
\toprule
Role & Method & Forget & Retain & Test & Avg. Diff & Time (s) \\
\midrule
Preprocessing & Official full AMUN & 29.47 & 48.79 & 26.62 & 30.688 & 1153.8 \\
 & Raw-forward fixed AMUN & 86.89 & 99.04 & 55.57 & 19.072 & 1527.6 \\
\midrule
Direct & No perturbation & 13.97 & 14.15 & 11.65 & 46.173 & 1.6 \\
 & DPL direct & 14.20 & 14.41 & 11.76 & 46.011 & 210.0 \\
\midrule
Utility & Clean retain CE & 44.60 & 49.80 & 34.59 & 26.396 & 7.3 \\
 & DPL utility & 38.79 & 43.26 & 30.78 & 28.848 & 143.1 \\
\midrule
Warm start & Official init, 3 steps & 88.17 & 99.18 & 55.60 & 19.866 & 1056.3 \\
 & FGSM init, 3 steps & 88.14 & 99.11 & 55.69 & 20.068 & 1039.1 \\
 & DPL init, 3 steps & 88.16 & 99.16 & 56.05 & 20.113 & 1262.8 \\
\bottomrule
\end{tabular}
}
\end{table}

\paragraph{Preprocessing diagnostic.} The preprocessing mismatch substantially changes the Tiny ImageNet results. The source checkpoint obtains 57.59\% validation accuracy under its normalized source-training transform but only 12.28\% under the common raw evaluator. Correspondingly, restoring normalization before the attack forward changes full-AMUN retain and test accuracy from 48.79\% and 26.62\% to 99.04\% and 55.57\%, and reduces Avg. Diff from 30.688 to 19.072. This is too large to treat preprocessing as an incidental implementation detail.

\paragraph{Role-specific results.}
Direct DPL improves Avg.~Diff.\ over the no-perturbation control by 0.162 and test accuracy by 0.11 percentage points, but both final models remain in the preprocessing-inconsistent low-accuracy regime, with test accuracy below 12\%. DPL also costs 210.0 seconds rather than 1.6 seconds. We therefore classify the direct result as inconclusive, not as a cross-dataset success. The utility result is clearer: DPL utility lowers retain and test accuracy by 6.54 and 3.81 points, worsens Avg. Diff by 2.452, and increases runtime from 7.3 to 143.1 seconds relative to clean retain CE. The warm-start result is also consistent with CIFAR-10. DPL raises test accuracy by 0.36--0.45 points, but has no Avg.-Diff advantage and increases total time to 1262.8 seconds, compared with 1039.1 seconds for FGSM and 1056.3 seconds for official initialization.

On Tiny ImageNet, neither the utility nor the warm-start comparison favors DPL. The direct comparison remains inconclusive because both methods use a preprocessing-inconsistent low-accuracy pipeline. We therefore treat Tiny ImageNet as an external implementation check and base the three role decisions on the CIFAR-10 experiments.

\section{Takeaways and Recommendations}
\label{sec:takeaways}

\paragraph{Match the control to the proposed role.}
A direct forgetting mechanism should be compared with retraining, no-perturbation updates, and simple first-order controls. A utility term should be compared with clean retain CE, KL anchoring, or distillation. An attack initialization should be compared with the official initialization and with zero, random, or FGSM initialization when appropriate. Step counts, gradient evaluations, and total runtime should be matched.

\paragraph{Evaluate behavior relative to retraining.}
Under random deletion, forgotten examples remain samples from the training distribution, so their retrained accuracy can remain high. Report absolute gaps from retraining in forget, retain, test, and membership-inference statistics. Immediate loss separation is only a diagnostic; follow it with a controlled model update and an evaluation of final behavior. Intermediate checkpoints can reveal effects that ordinary fine-tuning later erases.

\paragraph{Audit the implemented perturbation.}
A perturbation audit must state the coordinates used to compute gradients and apply perturbations, including spatial augmentation, normalization, projection, clipping, and realized perturbation norms. For label perturbations, report target KL divergence, argmax changes, and numerical precision. Treat preprocessing as part of the predictor, and use compatible transforms for source training, unlearning, retraining, and final evaluation.

\paragraph{Report the full procedure and its cost.}
Influence-derived directions require guide gradients, Hessian-vector products, and candidate mixed derivatives. Include these costs in the reported unlearning or attack time. A reproducible comparison must also identify the checkpoint, forget indices, retraining reference, guide and candidate sets, direction sign and normalization, raw and applied norms, matched controls, intermediate and final metrics, membership-inference configuration, and the predeclared threshold for practical improvement.

\section{Limitations and Broader Impact}
\label{sec:limitations}
\paragraph{Experimental scope.}
Our main CIFAR-10 validation uses three paired source/forget seeds with frozen configurations. Across these runs, the direct and warm-start decisions from seed 1 are stable, whereas the utility decision is not: one of three seeds passes the utility gate. Three seeds cannot separate a small real effect from seed noise in that role, and a broader study would be needed to characterize variability across training initializations and deletion requests. We report the paired outcomes together with descriptive sample standard deviations. The Tiny ImageNet experiment uses a single source--forget pair. Both datasets also focus on image classification with random 10\% instance deletion. Other deletion settings, including class deletion and difficult or peripheral forget sets, as well as generative and language models, may exhibit different behavior~\citep{zhao2024hard}.
\paragraph{Approximate unlearning.}
The procedures studied here are approximate unlearning methods rather than certified deletion mechanisms. Our accuracy and membership-based audits measure agreement with a matched retraining reference, but they cannot establish complete removal of information associated with forgotten samples. For applications that require formal deletion guarantees, empirical agreement with retraining would need to be complemented by stronger auditing or certified mechanisms.
\paragraph{Influence approximation.}
The influence calculation uses a damped inverse-Hessian vector product with a finite number of conjugate-gradient iterations in a nonconvex network, and the analysis in \cref{sec:analysis} is local. Extending the diagnostic solve from 20 to 40 and 80 CG iterations changes the relative residual only slightly (0.985, 0.984, and 0.980), while the cosine between the 20- and 80-iteration input directions is 0.219. The tested solve is therefore far from convergence, and the direction evaluated in \cref{sec:experiments} is the one produced by the released solver at its default budget rather than a faithful estimate of the influence direction in \cref{prop:retrain-direction}. Our negative results should be read at that level: they show that this computable direction does not fill the three roles, not that a converged influence direction could not. Increasing the CG budget from 20 to 80 iterations did not monotonically improve immediate forget-minus-retain loss separation, but three budgets are too few to draw conclusions about the converged limit. Alternative curvature approximations, such as spectrally guided LiSSA~\citep{klochkov2024ihvp}, and layer-restricted variants may lead to different behavior and are natural directions for further study.
\paragraph{Tuning breadth.}
The DPL search covered direction sign, signed versus normalized representation, $\alpha\in\{1,2,4,8\}/255$, and CG budget on seed 1, with the best configuration selected optimistically and then frozen. The utility weight $\lambda_u=0.1$, the warm-start initial norm $2/255$, and the guide and candidate set sizes were fixed and not tuned for DPL, and no architecture- or dataset-specific hyperparameters were adjusted for Tiny ImageNet. A wider search could favor DPL in the utility and warm-start roles, where the observed differences are small; it is unlikely to change the direct-role outcome, where the gap to SalUn and AMUN is large.
\paragraph{Implementation-specific evaluation details.}
The original seed-1 CIFAR-10 warm-start sweep follows the normalization and clamping behavior of the matched repository. In the final three-seed validation, we additionally use a strictly matched raw-space attack path and recover the same qualitative warm-start conclusion. For Tiny ImageNet, the official evaluation path omits source-training normalization before the model forward pass, so we report both the official path and a raw-forward-fixed variant. Exact online Tiny ImageNet RMIA could not be reproduced because the corresponding 128-model reference matrix was unavailable.
\paragraph{Additional settings.}
The inactive label intervention is observed under the float32 target scale used in our experiments. Rescaled soft targets, temperature calibration, or higher-precision implementations may change this behavior. We also leave several complementary directions open, including gradient-only influence surrogates, game-based unlearning quality metrics~\citep{tu2025reliable}, per-sample likelihood audits~\citep{naderloui2025rectifying}, and class-wise or worst-case forget sets. Evaluating these settings would help establish how broadly the empirical patterns reported here extend beyond the current benchmark configuration.
\paragraph{Broader impact.}
The experiments use standard image-classification benchmarks and do not involve newly collected personal data or human participants. Machine unlearning is often motivated by privacy rights, but an empirical unlearning score should not be represented to users as proof of legal or semantic erasure. The main ethical recommendation of this study is conservative reporting: distinguish deletion guarantees from behavioral audits, expose failed controls, and avoid converting implementation artifacts into privacy claims.
\section{Conclusion}
\label{sec:conclusion}

We evaluated influence-derived data perturbations as a direct forgetting mechanism, a retained-utility regularizer, and an initialization for adversarial unlearning. A mixed input-parameter derivative can define a first-order direction for matching an approximate retraining displacement, but this justification is local and requires consistent perturbation coordinates.

In the seed-1 diagnostic, direct DPL remains farther from exact retraining than SalUn and AMUN. With frozen configurations across three paired CIFAR-10 source/forget seeds, direct DPL fails the matched direct-role gate in every seed. Retained-utility regularization passes in one seed and does not provide a reliable advantage over clean retained-data cross-entropy. Under the strict raw-space comparison, DPL warm starting fails to outperform both official and FGSM initialization in every seed once direction-construction cost is included. The tested label perturbation produces no measurable float32 target change.

The one-seed Tiny ImageNet check also provides no advantage for utility regularization or warm starting. Its direct-role comparison is inconclusive because the source model and downstream evaluator use inconsistent preprocessing. These experiments concern random instance deletion and do not establish failure in other deletion regimes. Within this scope, evaluate an influence-derived direction at the endpoint of its claimed role and against the corresponding matched control.

\bibliographystyle{tmlr}
\bibliography{references}

\appendix
\section{Additional Experimental Details}
\label{app:details}

\subsection{Baseline output organization}
The matched codebase uses the same original model checkpoints, forget-index files, retained and test splits, and evaluation entry points for retraining, FT, GA, SalUn, AMUN, and DPL. The detailed mechanistic setting is source seed 1 and forget-index seed 1; the final frozen validation additionally uses paired source/forget seeds $(10,10)$ and $(100,100)$ with exact-seed retraining references. Each random 10\% request contains 5,000 CIFAR-10 training examples, leaving 45,000 retained examples.

The wider aggregate in \cref{tab:baseline-screen} pools the nine source-seed and forget-index combinations available for approximate baselines. Retraining is pooled over three source seeds because the retrained checkpoint depends on the retained split associated with the matched protocol. We do not attach standard errors to this table because the combinations share source checkpoints and are not independent replicates.

\subsection{Direct diagnostic configurations}
The diagnostic sweep tested image-only, label-only, combined image and label, and iterative recomputation configurations. Candidate and guide caps were 5,000, maximum CG iterations were 20, and the short-update learning rate was 0.003. The selected ``Best DPL direct'' row in \cref{tab:direct} is the checkpoint with the lowest average discrepancy among the completed direct diagnostic checkpoints. This selection is favorable to DPL and should not be interpreted as a held-out model choice.

\subsection{Direction validation subsets}
The immediate and short-update direction checks use fixed 1,024-example forget, retain, and test diagnostic subsets. Direction construction still uses the full 5,000 guide and candidate sets. We save the diagnostic indices and direction metadata so that positive and negative directions use the same examples.

\subsection{Tiny ImageNet external-check protocol}
The Tiny ImageNet check uses VGG-19 with batch normalization, source seed 1, forget-index seed 1, and 10,000 forgotten examples from a 100,000-example training set. Source and exact-retrain training use 201 epochs of SGD with initial learning rate 0.1, momentum 0.9, weight decay $5\times10^{-4}$, and a step scheduler with period 40 and factor 0.1. The matched approximate-unlearning arms start from the source checkpoint and use the repository's 10-epoch Tiny ImageNet recipes. DPL signs, raw-space scales, CG iterations, utility weight, and the three-step warm-start budget are fixed before this check and are not tuned on Tiny ImageNet.

Source training applies random crop, horizontal flip, and channel normalization, and validation applies the same normalization without spatial augmentation. The official downstream Tiny ImageNet path uses raw tensors for model forward. The raw-forward-fixed AMUN arm changes only this point by applying the source-training normalization differentiably before model forward; attack projection and clipping remain in raw $[0,1]$ coordinates. The source training process was resumed from complete periodic checkpoints after external interruptions. Model and optimizer state, schedule, split, and epoch count were preserved, but the augmentation and sampler RNG stream is not bitwise identical to an uninterrupted run. Exact online RMIA is omitted because the required official 128-model reference matrix is unavailable.

\subsection{Membership statistics}
The SVC confidence statistic is inherited from the AMUN evaluation pipeline. RMIA FT-AUC compares the forget and test distributions under the available reference-model protocol. We do not interpret either statistic monotonically in isolation. We compare each statistic with the exact retrained reference and include it in average discrepancy.

\section{Predeclared CIFAR-10 Gates}
\label{app:gates}

\subsection{Direct role}
The direct role was considered promising only if it moved forget behavior toward retraining while maintaining retain accuracy at least 99\%, test accuracy at least 93.5\%, positive forget-retain separation, and average discrepancy below the matched AMUN reference. No corrected configuration met these conditions.

\subsection{Utility role}
The DPL utility arm had to satisfy all of the following relative to clean retain CE: forget accuracy within 0.25 percentage points, no decrease in retain accuracy, no decrease in test accuracy, a strict improvement in at least one utility metric, and no worse SVC confidence gap to retraining. In the detailed seed-1 comparison it failed because retain accuracy decreased and retain loss increased. Under the frozen three-seed validation, seed 10 passes while seeds 1 and 100 do not.

\subsection{Warm-start role}
Relative to full official AMUN, a DPL few-step arm in the original seed-1 sweep required forget accuracy within 0.25 points, retain and test accuracy within 0.10 points, average discrepancy and SVC gap within 0.15, and at least 30\% total runtime reduction. It also had to beat official and FGSM initialization at the same number of steps by a practical threshold. No DPL arm passed all quality conditions. The final frozen validation separately repeats the three-step initialization comparison in strict raw-image coordinates for all three paired seeds; no seed passes against both simple initializations.

\section{Final Cross-Seed Validation Diagnostics}
\label{app:final-validation}

\subsection{Strict raw-space warm-start metrics}
\Cref{tab:rawspace-warm-app} gives the complete three-step raw-space comparison used in the final warm-start decision. RMIA FT-AUC is reported directly rather than interpreted monotonically; the matched exact-seed reference is used for the role gate.

\begin{table}[h]
\caption{Strict raw-space three-step warm-start validation. Total time includes DPL direction construction when applicable.}
\label{tab:rawspace-warm-app}
\centering
\scriptsize
\resizebox{\textwidth}{!}{\begin{tabular}{rllrrrrrr}
\toprule
Seed & Init. & Forget & Retain & Test & SVC gap & RMIA FT-AUC & Avg. Diff & Time (s) \\
\midrule
1 & Official & 95.44 & 99.60 & 93.37 & 0.84 & 0.4997 & 0.779 & 181.9 \\
1 & FGSM & 96.02 & 99.64 & 93.23 & 1.14 & 0.5006 & 1.026 & 189.7 \\
1 & DPL & 95.84 & 99.60 & 93.50 & 1.34 & 0.5001 & 0.974 & 244.9 \\
\addlinespace
10 & Official & 95.04 & 99.63 & 93.43 & 0.22 & 0.5017 & 0.408 & 178.5 \\
10 & FGSM & 94.76 & 99.61 & 93.38 & 0.50 & 0.5003 & 0.515 & 183.9 \\
10 & DPL & 94.84 & 99.55 & 93.26 & 0.12 & 0.5023 & 0.444 & 231.1 \\
\addlinespace
100 & Official & 95.60 & 99.60 & 93.33 & 2.20 & 0.5006 & 1.114 & 175.9 \\
100 & FGSM & 95.26 & 99.67 & 93.55 & 0.98 & 0.4966 & 0.651 & 178.5 \\
100 & DPL & 95.52 & 99.62 & 93.57 & 1.30 & 0.4994 & 0.804 & 238.3 \\
\bottomrule
\end{tabular}
}
\end{table}

\subsection{Direction alignment}
The coordinate-valid corrected normalized and signed directions have mean per-sample cosine 0.653 and median 0.659 on 1,024 examples; 99.7\% exceed cosine 0.5. The original legacy-to-corrected comparison is undefined because the original execution did not persist the legacy direction tensor and per-sample crop/flip metadata needed for a raw-coordinate pullback. We do not impute this missing comparison.

\begin{table}[h]
\caption{Coordinate-valid corrected-direction alignment. The legacy comparison is omitted because it cannot be reconstructed without missing augmentation metadata.}
\label{tab:alignment-app}
\centering
\small
\begin{tabular}{lrrrrr}
\toprule
Comparison & Global cos. & Mean cos. & Median cos. & $P(\cos>0.5)$ & $n$ \\
\midrule
Normalized vs.\ signed corrected direction & 0.653 & 0.653 & 0.659 & 0.997 & 1024 \\
\bottomrule
\end{tabular}

\end{table}

\subsection{CG stability}
The fixed 20-iteration CG solve is not numerically converged under the diagnostic used here. Increasing the budget changes the recovered direction substantially, but does not monotonically strengthen the immediate forget-minus-retain loss separation.

\begin{table}[h]
\caption{CG stability diagnostic. Cosines are measured against the 80-iteration diagnostic solution; longer runs are diagnostics only and are not used to reselect the final DPL configuration.}
\label{tab:cg-app}
\centering
\scriptsize
\resizebox{\textwidth}{!}{\begin{tabular}{rrrrrrr}
\toprule
CG iters & Time (s) & Rel. residual & Soln.\ cos. vs.\ 80 & Input cos. vs.\ 80 & Loss sep. & HVP calls \\
\midrule
20 & 49.8 & 0.985 & 0.325 & 0.219 & 0.0488 & 840 \\
40 & 94.9 & 0.984 & 0.628 & 0.409 & 0.0187 & 1640 \\
80 & 181.6 & 0.980 & 1.000 & 1.000 & 0.0267 & 3240 \\
\bottomrule
\end{tabular}
}
\end{table}

\subsection{Reproducibility records}
The manuscript archive includes the final-validation summary CSVs and the supplied audit, runner, status, and artifact-manifest records under \texttt{data/} and \texttt{supplementary/final\_validation/}. These records identify the paired seeds, configuration hashes, checkpoint hashes, forget-index hashes, commands, and result files used for the tables above. Raw model checkpoints and the reference-model bank are not duplicated in the manuscript archive.

\section{Additional Numerical Audit}
\label{app:numerics}

\subsection{Label direction statistics}
The full candidate label direction had shape $5000\times10$, minimum $-3.266986\times10^{-5}$, maximum $2.227792\times10^{-5}$, standard deviation $2.619885\times10^{-6}$, $\ell_2$ norm $5.85818\times10^{-4}$, and $\ell_\infty$ norm $3.266986\times10^{-5}$. At $\beta=0.007$, the measured target KL and target-coordinate changes were zero. At $4\beta$, the largest target-coordinate change was approximately $1.06\times10^{-8}$ and the KL was no larger than $1.03\times10^{-11}$.

\subsection{Utility direction statistics}
The utility direction used 5,000 nonzero candidate perturbations. Its raw aggregate norms were $\ell_1=25676.99$, $\ell_2=61.263$, and $\ell_\infty=5.976$. After normalization and raw-space application, the aggregate perturbation norms were $\ell_1=1362.45$, $\ell_2=0.554$, and $\ell_\infty=0.00273$. The fraction of pixels clipped below zero was 0.118\%, and the fraction clipped above one was 0.493\%.

\subsection{Warm-start coordinate diagnostic}
The DPL raw initialization had mean per-sample $\ell_2$ norm $2/255=0.007843$. After division by CIFAR standard deviations, its mean attack-space norm was about 0.0391 under the repository convention. The subsequent clamp of normalized tensors to $[0,1]$ produced a much larger apparent post-projection distance. The same issue affected the FGSM path. We therefore treat the warm-start experiment as a matched repository comparison, not a strict raw-space threat-set study.


\end{document}